\documentclass{article}
\usepackage{ijcai25}

\usepackage{times}
\usepackage{soul}
\usepackage{url}
\usepackage[hidelinks]{hyperref}
\usepackage[utf8]{inputenc}
\usepackage[small]{caption}
\usepackage{graphicx}
\usepackage{amsmath}
\usepackage{amsthm}
\usepackage{booktabs}
\usepackage{algorithm}
\usepackage{algorithmic}
\usepackage[switch]{lineno}
\usepackage{amssymb}
\usepackage{hyperref}
\usepackage{booktabs}
\usepackage{multirow}
\usepackage{subfigure}
\usepackage{bm}
\usepackage{dsfont}
\usepackage{xcolor}
\usepackage[T1]{fontenc}

\newtheorem{proposition}{Proposition}
\newtheorem{lemma}{Lemma}

\title{LoD: Loss-difference OOD Detection by Intentionally Label-Noisifying Unlabeled Wild Data}

\author{
Chuanxing Geng$^{1,2,3}$
\and
Qifei Li$^{1}$\and
Xinrui Wang$^{1}$\and
Dong Liang$^{1,3}$\and
Songcan Chen$^{1,3}$\And
Pong C. Yuen$^2$\footnote{Corresponding author}\\
\affiliations
$^1$College of Computer Science and Technology, Nanjing University of Aeronautics and Astronautics\\
$^2$Department of Computer Science, Hong Kong Baptist University\\
$^3$MIIT Key Laboratory of Pattern Analysis and Machine Intelligence\\
\emails
\{gengchuanxing, liqifei, wangxinrui, liangdong, s.chen\}@nuaa.edu.cn,
pcyuen@comp.hkbu.edu.hk
}

\begin{document}

\maketitle

\begin{abstract}
    Using unlabeled wild data containing both in-distribution (ID) and out-of-distribution (OOD) data to improve the safety and reliability of models has recently received increasing attention. Existing methods either design customized losses for labeled ID and unlabeled wild data then perform joint optimization, or first filter out OOD data from the latter then learn an OOD detector. While achieving varying degrees of success, two potential issues remain: (i) Labeled ID data typically dominates the learning of models, inevitably making models tend to fit OOD data as IDs; (ii) The selection of thresholds for identifying OOD data in unlabeled wild data usually faces dilemma due to the unavailability of pure OOD samples. To address these issues, we propose a novel loss-difference OOD detection framework (LoD) by \textit{intentionally label-noisifying} unlabeled wild data. Such operations not only enable labeled ID data and OOD data in unlabeled wild data to jointly dominate the models' learning but also ensure the distinguishability of the losses between ID and OOD samples in unlabeled wild data, allowing the classic clustering technique (e.g., K-means) to filter these OOD samples without requiring thresholds any longer. We also provide theoretical foundation for LoD's viability, and extensive experiments verify its superiority.
\end{abstract}

\section{Introduction}
The safety and reliability of traditional machine learning models often face challenges when deployed in real-world environments due to unexpected occurrence of out-of-distribution (OOD) data \cite{nguyen2015deep}. To meet this challenge, the OOD detection problem has been studied \cite{hendrycks2016baseline,yang2024generalized}, which requires the models not only predict the true class of in-distribution (ID) data but also effectively reject the OOD data. To date, numerous OOD detection methods have been developed \cite{liu2020energy,abati2019latent,wang2022vim,hendrycks2018deep,katz2022training}, and among them, the methods leveraging unlabeled wild data containing ID and OOD samples to improve the performance of OOD detection has recently received increasing attention \cite{katz2022training}. This mainly attributed to the fact that such data can be freely collected during the deployment of any machine learning model in its operational environment, while also allowing for the capture of the true test-time OOD distribution.

Despite the promise, harnessing the power of unlabeled wild data is non-trivial due to the heterogeneous mixture of ID and OOD samples. Existing methods either adopt a joint optimization strategy \cite{katz2022training} or a two-step strategy (i.e., filtering and learning) \cite{du2024does}. The former aims to design customized losses for labeled ID and unlabeled wild data to jointly optimize the models in a semi-supervised learning manner. The latter first filters out OOD samples from the unlabeled wild data using customized OOD score (usually based on labeled ID data), then uses them along with labeled ID data to learn an OOD detector. While achieving varying degrees of success, two potential issues of these methods remain:
\begin{itemize}
    \item[\checkmark] \textbf{The model-bias issue.} Labeled ID data typically dominates the model learning in both two strategies, especially for the two-step strategy, thus inevitably making the model tend to fit OOD data as IDs.
    \item[\checkmark] \textbf{Threshold selection dilemma.} The selection of thresholds for determining OOD samples in unlabeled wild data usually faces challenges due to the unavailability of pure OOD samples.
\end{itemize}

To address these issues, this work proposes a novel loss-difference OOD detection framework (abbreviated as LoD) by \textit{intentionally label-noisifying} unlabeled wild data. LoD adopts the filtering and learning strategy and its key lies in the loss-difference filtering module with \textit{intentional label-noises}. In this module, the whole unlabeled wild data is intentionally labeled as a single $K+1$-th class (assuming that ID data contain $K$ classes), and then trained together with the labeled ID data through the \textit{fully-supervised} manner of $K+1$ classification. We would like to emphasize that such operations ingeniously transform the OOD filtering problem in unlabeled wild data into a label-noise learning problem, allowing us to solve the aforementioned issues by leveraging the inherent properties in label-noise learning. In this way, the OOD samples in unlabeled wild data is intentionally transformed into \textit{label-clean} samples, while the ID counterparts become \textit{label-noise} ones. The former naturally and seamlessly enables OOD samples in unlabeled wild data to jointly dominate the model learning with the labeled ID data, effectively addressing the model-bias issue. Meanwhile, the latter provides the key clues for differentiating ID and OOD samples in unlabeled wild data due to the significant differences in the loss curves between ID (\textit{label-noise}) and OOD (\textit{label-clean}) samples during training.

\begin{figure}[]
  \centering
  \subfigure{\includegraphics[width=4.23cm, height=3.6cm]{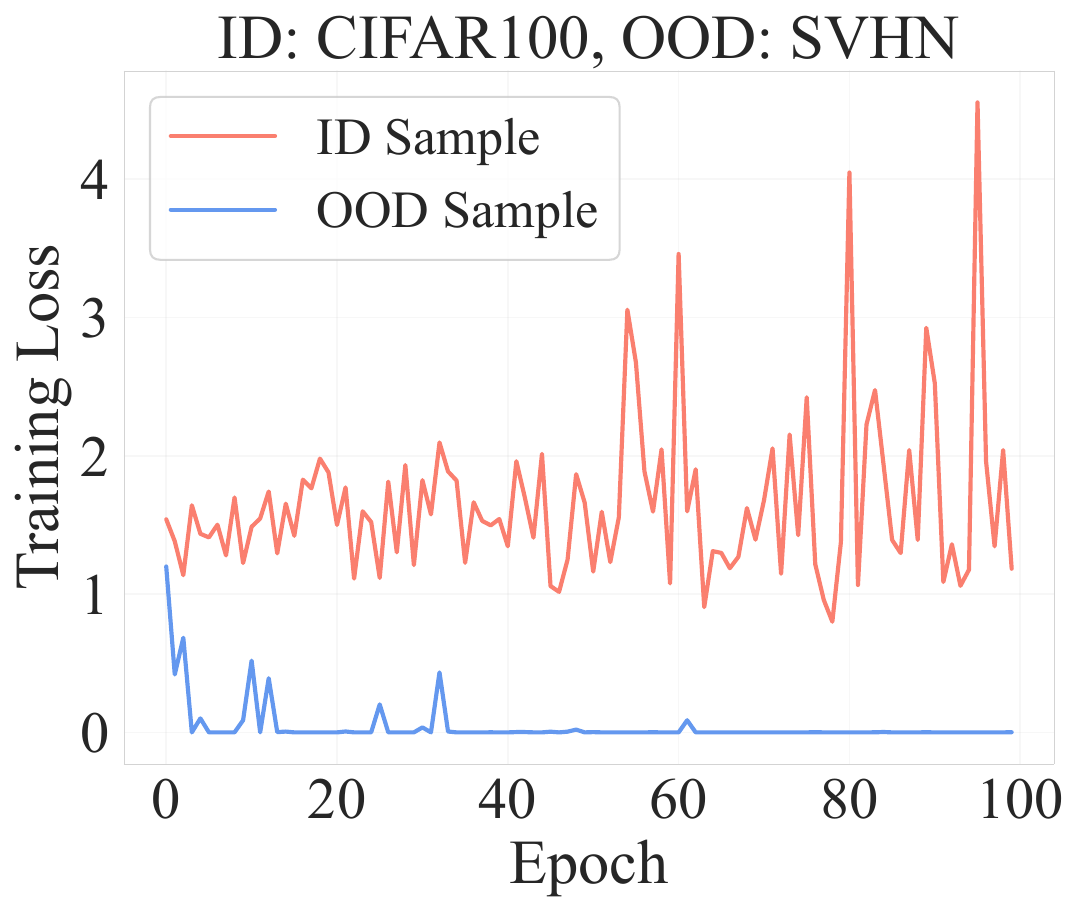}}
  \hfill
  \subfigure{\includegraphics[width=4.23cm, height=3.6cm]{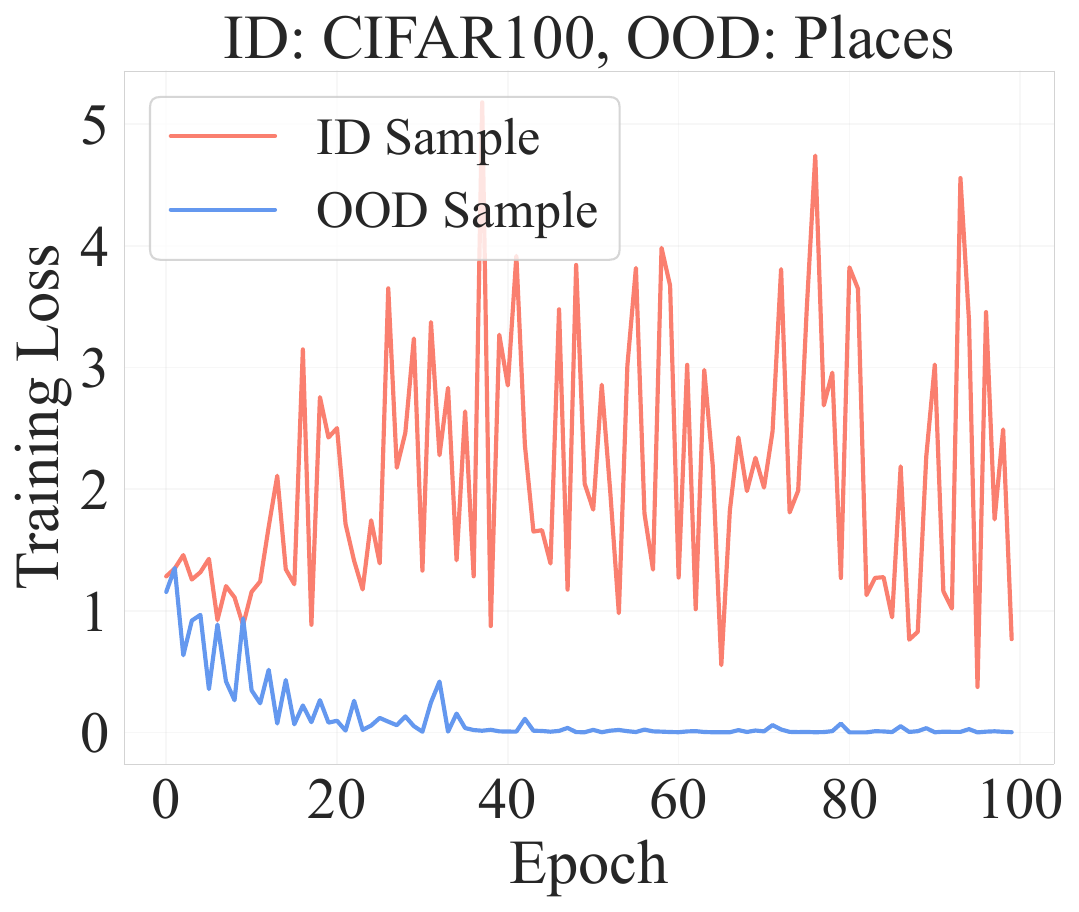}}
  
  \caption{The cross-entropy loss changes of ID (\textit{label-noise}) and OOD (\textit{label-clean}) samples in unlabeled wild data when they are intentionally labeled as $K+1$-th class. These two types of samples typically exhibit different loss curves due to the differences in how learning progresses for each.
  }
  \label{fig:fig1}
\end{figure}

As shown in Figure \ref{fig:fig1}, as the OOD samples in the unlabeled wild data are correctly labeled (\textit{label-clean}), the model fits them well as learning progresses, leading to a gradual decrease and convergence of the loss curve. In contrast, the corresponding ID part, due to being incorrectly labeled (\textit{label-noise}) and conflicting with the originally labeled ID data, exhibits not only higher loss values but also larger fluctuations during training. Such significant and natural differences allow us to employ classic clustering models, like K-means, to filter these OOD samples without requiring thresholds any longer. In particular, we also provide theoretical foundation to support the viability of such a module. Overall, our contributions can be highlighted as follows:
\begin{itemize}
    \item Two potential issues (i.e., the model-bias issue and threshold selection dilemma) in this OOD research line are identified, providing some new insights for the subsequent modeling of OOD detection.
    \item The OOD filtering problem in unlabeled wild data is elegantly reformulated as a label-noise learning problem, leading to a novel LoD OOD detection framework, which not only effectively addresses the model-bias issue but also circumvents the threshold selection dilemma.
    \item Theoretical foundation is provided to support the viability of LoD. Meanwhile, extensive experiments are also conducted to demonstrate its superiority. 
\end{itemize}


\section{Related Works}
\subsection{Out-of-Distribution Detection}
To improve the safety and reliability of models in detecting OOD data, various OOD methods have been developed \cite{zhu2023diversified,zheng2023out,wang2023out,yang2024generalized,li2024learning,behpour2024gradorth,fang2024learnability,sharifi2025gradient}, including adopting the classification confidence or entropy, modeling the ID density, leveraging auxiliary OOD data, and more. Among these, methods using auxiliary OOD data have demonstrated encouraging OOD detection performance over the counterpart without auxiliary data \cite{lee2017training,bevandic2018discriminative,malinin2018predictive,liu2020energy,chen2021atom,wei2022mitigating,du2022unknown,wang2023learning,sharifi2025gradient}. Despite the promise, there are two primary limitations: First, such data may not match the true distribution of OOD data in the wild; Second, collecting such data can be labor-intensive and inflexible. To address these limitations, recent works \cite{katz2022training,du2024does} proposed to leverage the unlabeled "in-the-wild" data due to they are freely collected during the deployment of any machine learning model in its operational environment, while also allowing for the capture of the true test-time OOD distribution. 

Our work falls into this research line, and as mentioned earlier, though the methods in this research line have achieved varying degrees of success, they still face two potential weaknesses, i.e., the model-bias issue and the threshold selection dilemma. These motivate us to seek new methods to address these issues.

\subsection{Training Neural Networks with Label Noises}
In many applications \cite{guan2018said}, due to the cost or difficulty of manual labeling, datasets are often annotated through online queries \cite{yuan2024combating} or crowdsourcing \cite{li2024certainty}. Such annotations inevitably contain numerous mistakes, i.e., label-noises. When trained on the data mixed clean labels and noise labels, deep neural networks have been observed to first fit label-clean data during an early learning phase, and then start memorizing the label-noise data after sufficient epochs of training \cite{liu2020early}. This phenomenon is independent of the optimizations used during training or the architectures of neural networks employed \cite{arpit2017closer}. In particular, during the early learning phase, label-clean and label-noise data will have different loss curves due to the difference in how learning progresses for each type. This has been exploited in many label-noise learning works \cite{forouzesh2022leveraging,li2023disc,yuan2024early,linlearning,lienen2024mitigating,yue2024ctrl}. For more information, please refer to the recent review work \cite{song2022learning}.

In this paper, we propose a novel loss-difference OOD detection framework by \textit{intentionally label-noisifying} unlabeled wild data, which interestingly transforms the OOD filtering problem in unlabeled wild data into a label-noise learning problem. This enables us to leverage the aforementioned inherent phenomenon of label-noise learning to effectively filter OOD data from the unlabeled wild data.

\begin{figure*}[]
  \centering
  \includegraphics[width=17.3cm,height=4cm]{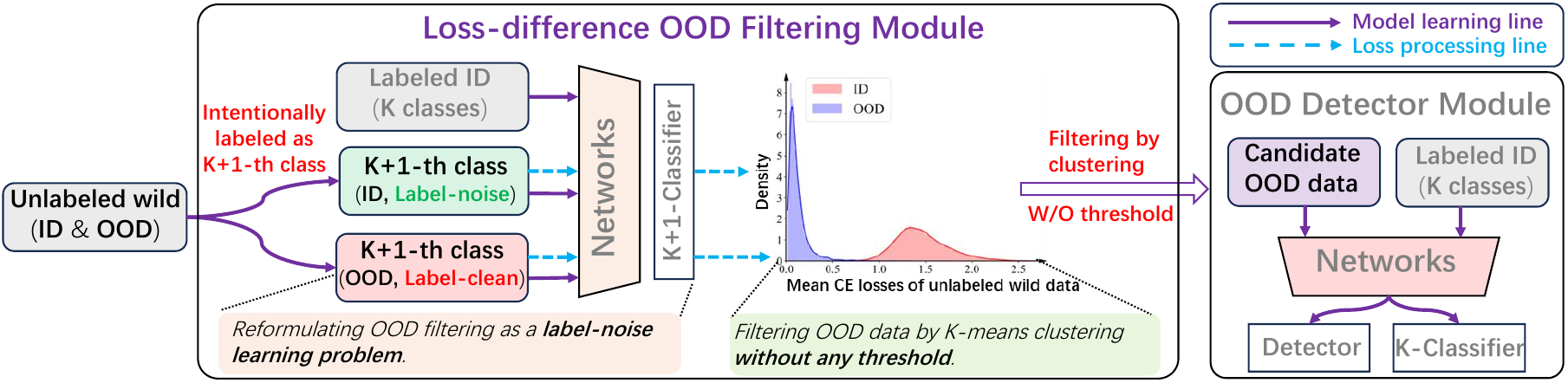}
  \label{fig:sensitivity}
  \caption{Overview of the loss-difference OOD detection framework by intentionally label-noisifying unlabeled
wild data.
  }
  \label{fig:fig2}
\end{figure*}

\section{Methodology}
\subsection{Problem Formulation}
\paragraph{Labeled ID Data} 
Let $\mathcal{X}$ denote the input space and $\mathcal{Y}=\{1,...,K\}$ represent the label space. Let $\mathcal{D}_{\text{in}}^{\text{train}}=\{(\bm{x}_i,y_i)\}_{i=1}^n$ denote the labeled training set drawn independently and identically from $\mathbb{P}_{\mathcal{XY}}$. $\mathbb{P}_{\text{in}}$ is the marginal distribution of $\mathbb{P}_{\mathcal{XY}}$ on $\mathcal{X}$, which is also referred to as the ID distribution.
\paragraph{Unlabeled Wild Data} 
The main challenge in OOD detection is the lack of labeled OOD data. In particular, the sample space for potential OOD data can be prohibitively large, making it expensive to collect labeled OOD data. To model the realistic environment, recent works \cite{katz2022training,du2024does} incorporated unlabeled wild data $\mathcal{D}_{\text{wild}}=\{\widetilde{\bm{x}}_1,...,\widetilde{\bm{x}}_m\}$ into OOD detection. Unlabeled wild data consists of potentially both ID and OOD data, and can be freely collected upon deploying an existing model in its natural habitats. Following \cite{katz2022training}, the Huber contamination model is employed to characterize the marginal distribution of the unlabeled wild data:
\begin{equation}
    \mathbb{P}_{\text{wild}} := (1-\pi)\mathbb{P}_{\text{in}} + \pi\mathbb{P}_{\text{out}},
\end{equation}
where $\pi\in(0,1]$, and $\mathbb{P}_{\text{out}}$ is the OOD distribution defined over $\mathcal{X}$.
\paragraph{Learning Goal}
The learning framework aims to build the OOD detector $g_\theta$ and the multi-class classifier $f_\theta$ by leveraging data from both $\mathcal{D}_{\text{in}}^{\text{train}}$ and $\mathcal{D}_{\text{wild}}$. Following \cite{du2024does}, we here are interested in the following measurements for model evaluation:
\begin{eqnarray*}
    &&\downarrow\text{FPR}(g_\theta):=\mathbb{E}_{\bm{x}\sim\mathbb{P}_{\text{out}}}(\mathds{1}\{g_\theta(\bm{x})=\text{in}\}),\\
    &&\uparrow\text{TPR}(g_\theta):=\mathbb{E}_{\bm{x}\sim\mathbb{P}_{\text{in}}}(\mathds{1}\{g_\theta(\bm{x})=\text{in}\})
\end{eqnarray*}


\subsection{Loss-Difference OOD Detection Framework}
To effectively address the two aforementioned potential issues, i.e., the model-bias issue and the threshold-selection dilemma, we innovatively propose a novel loss-difference OOD detection framework (abbreviated as LoD) by \textit{intentionally label-noisifying} unlabeled wild data. As shown in Figure 2, LoD follows the two-step strategy and contains two main modules, i.e., loss-difference OOD filtering module and OOD detector learning module. Next, we will elaborate on the specific details of each module.


\subsubsection{Loss-difference OOD Filtering Module}
In this part, a loss-difference filtering mechanism with \textit{intentional label-noises} is developed, which ingeniously reformulates the OOD filtering problem in unlabeled wild data as a label-noise learning problem with \textit{controllable label-noise ratio}. This allows us to leverage the inherent properties of label-noise learning demonstrated in Section 2.2 to effectively filter OOD data from the unlabeled wild data. 

In specific, we first intentionally label the whole unlabeled wild data as a single $K+1$-th class (assuming that ID data contains $K$ classes) and then train them together with labeled ID data in a \textit{fully-supervised manner of $K+1$ classification}, as follows:
\begin{equation}
\begin{aligned}
    &\mathcal{L} = \frac{1}{|\mathcal{B}_{\text{in}}^{\text{train}}|}\sum_{(\bm{x}_i,y_i)\sim\mathcal{B}_{\text{in}}^{\text{train}}}\ell(\hat{y_i},y_i) +  \\&\frac{1}{|\mathcal{B}_{\text{wild}}|}\sum_{(\bm{x}_i,y_i)\sim\mathcal{B}_{\text{wild}}}\ell(\hat{y_i},y_{K+1}), \ \  \ 
    \hat{y_i} = f(\bm{x}_i, \bm{\theta}),
\end{aligned}
\end{equation}
where $f(\cdot,\bm{\theta})\in \mathcal{F}$ denotes the $K+1$ classifier, $l(\cdot,\cdot)$ represents the vanilla cross-entropy (CE) loss. Each training batch consists of two parts: $\mathcal{B}_{\text{in}}^{\text{train}}$ and $\mathcal{B}_{\text{wild}}$, respectively sampled from labeled ID data and unlabeled wild data. Note that the ratio of $|\mathcal{B}_{\text{in}}^{\text{train}}|:|\mathcal{B}_{\text{wild}}|\geq 1$ is controllable. In fact, we indirectly control the label-noise ratio of the learning task by controlling this ratio (for more details, please refer to \underline{Section~4}). 

By labeling the entire unlabeled wild data as a single $K+1$-th class, the ID samples in $\mathcal{D}_{\text{wild}}$ are intentionally converted to \textit{label-noise} samples while the OOD samples in $\mathcal{D}_{\text{wild}}$ become \textit{label-clean} ones. According to the inherent phenomenon of early learning stage in label-noise learning, the loss curves of these two types of labeled samples will exhibit significant difference during the early learning stage, as shown in Figure \ref{fig:fig1}. This discrepancy provides us a critical clue for effectively distinguishing between them. Therefore, after training the $K+1$ classifier, we conduct clustering operations on the loss values of unlabeled wild data gained during training so as to filter the OOD samples from $\mathcal{D}_{\text{wild}}$, \textit{which does not need the filtering thresholds any more}. Particularly, clustering in our case has a well-defined number of clusters -- Two -- corresponding to the ID and OOD clusters represented by their distinct loss behaviors throughout the training process, e.g., higher loss-values for ID data while lower counterparts for OOD data.


Considering the efficiency issue, we here utilize the mean of loss-values during training as the new features for each sample in $\mathcal{D}_{\text{wild}}$. Then the classic K-means clustering technique is employed to achieve the OOD samples filtering from unlabeled wild data. Let $\mu_1$, $\mu_2$ ($\mu_1>\mu_2$) respectively denote the ID and OOD cluster centers, while $d_1$, $d_2$ respectively denote the distances between the corresponding sample and the two cluster centers. We filter the OOD samples from the unlabeled wild data by the following rule:
\begin{equation}
    \hat{y} = \left\{
    \begin{array}{ll}
       \text{ID data}, \ \ \text{if} \ \ d_1<d_2,   &  \\
       \text{OOD data}, \ \  \text{otherwise}. & 
    \end{array}
    \right.
\end{equation}

\begin{figure}[]
  \centering
  \subfigure{\includegraphics[width=4cm, height=3.2cm]{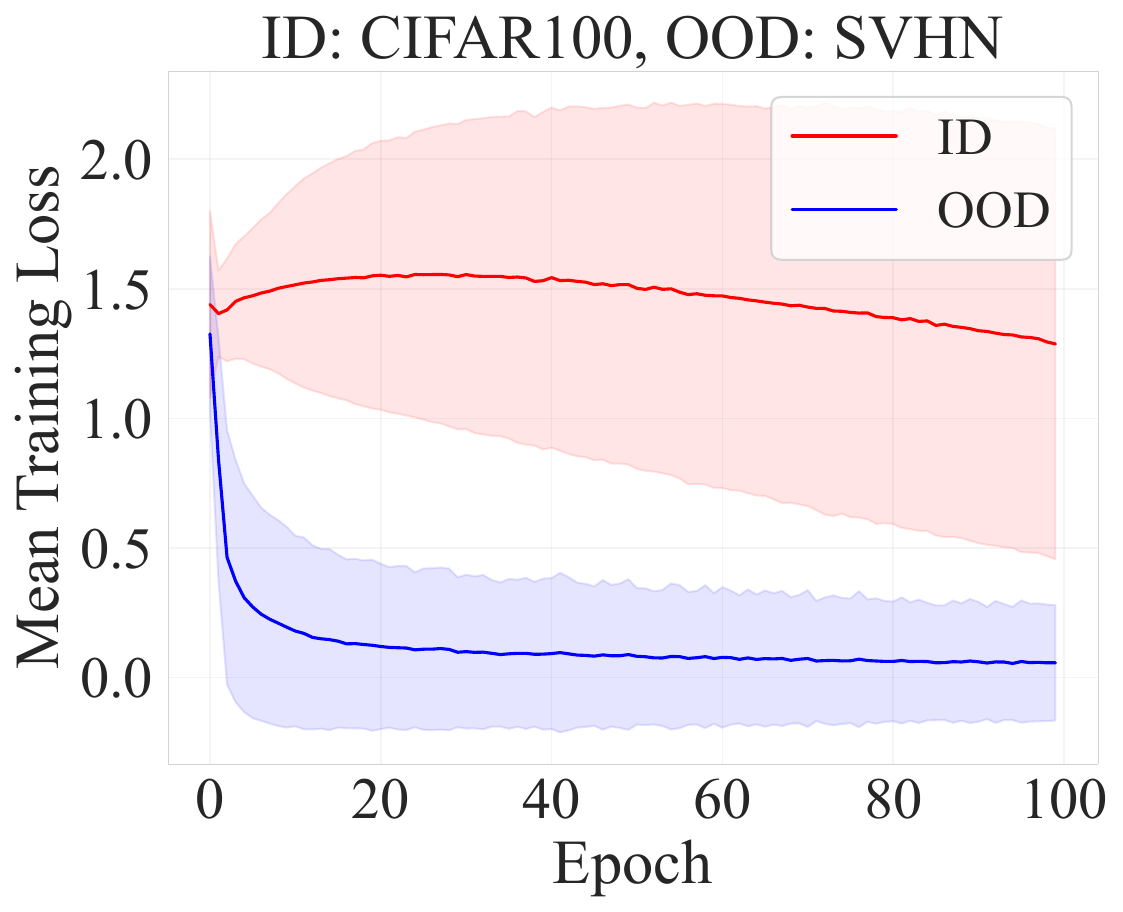}}
  \hfill
  \subfigure{\includegraphics[width=4cm, height=3.2cm]{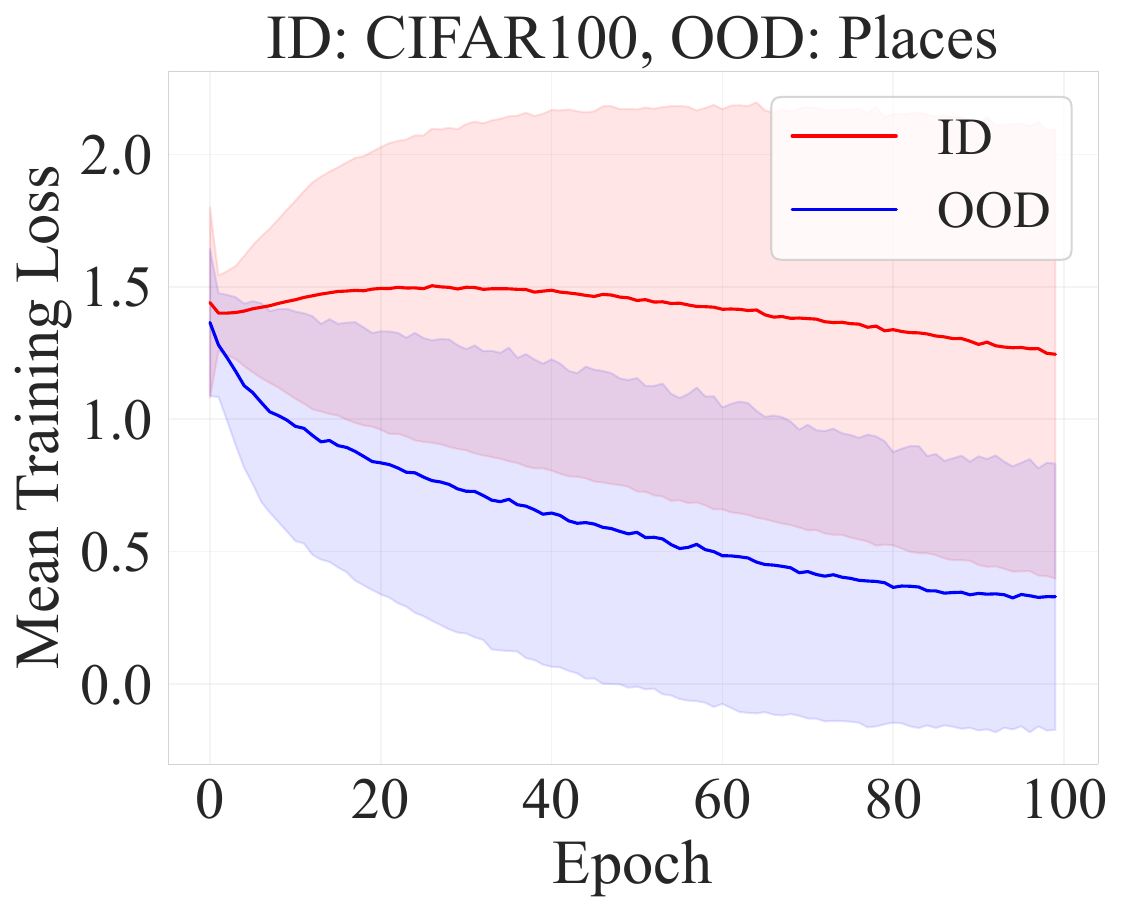}}
  \caption{The mean cross-entropy loss curves respectively for all ID (\textit{label-noise}) and OOD samples (\textit{label-clean}) in unlabeled wild data when they are intentionally labeled as $K+1$-th class.
  }
  \label{fig:fig3}
\end{figure}

\paragraph{Remark.} At first glance, labeling the entire set of unlabeled wild data as a single $K+1$-th class seems potentially to undermine the model learning.  Intriguingly, however, once we switch to consider filtering the OOD samples from the label-noise perspective, such operations, just on the contrary, bring at least the following three-fold advantages:
\begin{itemize}
    \item \textbf{First}, OOD samples in unlabeled wild data are correctly labeled (\textit{label-clean}), naturally and seamlessly enabling them to jointly dominate the model learning with labeled ID data, thus effectively circumventing the model-bias issue.
    \item \textbf{Second}, as mentioned earlier, the ID samples in $\mathcal{D}_{\text{wild}}$ being erroneously labeled (\textit{label-noise}) as $K+1$-th class contradict the label-correct ones in labeled ID data $\mathcal{D}_{\text{in}}^{\text{train}}$, thereby resulting in their loss curves exhibiting both higher values and greater fluctuations compared to those of OOD samples, such as their mean-loss curves shown in Figure \ref{fig:fig3}. This discrepancy provides us a fairly clear signal to distinguish ID and OOD samples in the unlabeled wild data.
    \item \textbf{Third}, our LoD is data-centric in nature, wherein we just relabel the unlabeled wild data as the intentionally $K+1$-th class without any modifications to the network architectures we employed. This endows our LoD stronger applicability (for related experiments, please refer to Appendix E).  
\end{itemize}

To solidly demonstrate the viability of this OOD filtering module, we also provide the theoretical analyses to support our first two claims, which will be detailed in Section 4.


\subsubsection{OOD Detector Learning Module}
After obtaining the candidate OOD samples $\mathcal{D}_{\text{out}}$ from the unlabeled wild data, we training an OOD detector $g_\theta$ using them together with labeled ID data $\mathcal{D}_{\text{in}}^{\text{train}}$. Similar to \cite{du2024does}, we adopt the following optimization objective:
\begin{equation}
    \mathcal{L}(g_\theta) = \mathbb{E}_{\bm{x}\in\mathcal{D}_{\text{in}}^{\text{train}}}\mathds{1}\{g_\theta(\bm{x})\leq0\} + \mathbb{E}_{\widetilde{\bm{x}}\in\mathcal{D}_{\text{out}}}\mathds{1}\{g_\theta(\widetilde{\bm{x}})>0\},
\end{equation}
where the binary sigmoid loss is employed as the smooth approximation of the $0/1$ loss to make it tractable. In addition, a $K$-class classifier $f_\theta$ is also trained using CE loss on labeled ID data along with $g_\theta$ to ensure the ID accuracy. Algorithm 1 denotes the entire workflow of our LoD.

\begin{algorithm}
\label{alg:algorithm1}
\newcommand{\MyComment}[1]{\textcolor{gray}{\textit{#1}}}
\footnotesize
\caption{LoD OOD Detection Framework}
\begin{algorithmic}[1]

    \renewcommand{\algorithmicrequire}{\textbf{Input:}}
    \renewcommand{\algorithmicensure}{\textbf{Output:}}
    
    \REQUIRE In-distribution data $\mathcal{D}^{train}_{in}$, unlabeled wild data $\mathcal{D}_{\text{wild}}$, Max Epoch $T$, Batch size $|\mathcal{B}|$.
    \ENSURE OOD detector $g_\theta$ and classifier $f_\theta$. 
    
    \STATE \MyComment{\# Loss-difference OOD detection module} 

    \STATE \textbf{Initializing:} Model parameters $\bm{\theta}$, $\mathcal{D}_{\text{wild}}$ labeled as $K+1$-th class, loss record matrix $\mathcal{V}=\{\}\in \mathbb{R}^{|\mathcal{D}_{\text{wild}}| \times T}$. 
    
    \FOR{$epoch=1$ to $T$}
        \STATE Batch $\mathcal{B}=\mathcal{B}_{\text{in}}^{\text{train}} \cup \mathcal{B}_{\text{wild}}$, where $\mathcal{B}_{\text{in}}^{\text{train}}$ samples from $\mathcal{D}_{\text{in}}^{\text{train}}$ and $\mathcal{B}_{\text{wild}}$ samples from $\mathcal{D}_{\text{wild}}$. 
        \STATE Update $K+1$ classifier $f(\cdot,\bm{\theta})$ based on Eq.(2). 

        \STATE Record losses of unlabeled wild data. $\mathcal{V} \leftarrow \mathcal{V} \cup \{ l_{i} \mid i \in (1, |\mathcal{B}_{\text{wild}}|)\}$

    \ENDFOR 
    
    \STATE Calculate the mean-loss set of wild data $\{u_i\} = \texttt{mean}(\mathcal{V})$, where $\{u_i\} \in \mathbb{R}^{|\mathcal{D}_{\text{wild}}| \times 1}$.
    \STATE Cluster and detect candidate OOD samples set $\mathcal{D}_{\text{out}}$ based on Eq.(3). 

    \STATE \MyComment{\# OOD detector learning module}
    \FOR{$epoch=1$ to $T$}
        \STATE Batch $\mathcal{B} = \mathcal{B}_{\text{in}}^{\text{train}} \cup \mathcal{B}_{\text{out}}$, where $\mathcal{B}_{\text{in}}^{\text{train}}$ samples from $\mathcal{D}_{\text{in}}^{\text{train}}$ and $\mathcal{B}_{\text{out}}$ samples from $\mathcal{D}_{\text{out}}$.
        \STATE Update  $f_\theta$ and $g_\theta$ based on Eq.(4).
    \ENDFOR
\end{algorithmic}
\end{algorithm}

\section{Theoretical Analysis}
\subsection{Mitigation of The Model Bias}
For the first claim in Subsection 3.2, we here provide a theoretical analysis at the gradient level to demonstrate that in our LoD framework, labeled ID data and OOD data in $\mathcal{D}_{\text{wild}}$ can jointly dominate the model learning. Let $N_1$ denote the number of samples in $\mathcal{B}_{\text{in}}$, while $N_2$ and $N_3$ denote the number of samples respectively from IDs and OODs in $\mathcal{B}_{\text{wild}}$. Then Eq.(2) can be rewritten in the following form:
\begin{equation}
\begin{aligned}
    \mathcal{L} &=  \underbrace{\frac{1}{N_1}\sum_{(\bm{x}_i,y_i)\sim\mathcal{B}_{\text{in}}}\ell(\hat{y_i},y_i) + \frac{1}{N_2}\sum_{(\bm{x}_i,y_i)\sim\mathcal{B}_{\text{wild}}}\ell(\hat{y_i},y_{K+1})}_{\text{ID data}} \\ &+ \underbrace{\frac{1}{N_3}\sum_{(\bm{x}_i,y_i)\sim\mathcal{B}_{\text{wild}}}\ell(\hat{y_i},y_{K+1})}_{\text{OOD data}}.
\end{aligned}
\end{equation}
Let $\nabla\mathcal{L}_{N_k} = \frac{1}{N_k}\sum_{i=1}^{N_k}\nabla l(\hat{y_i},y_i), k=1,2,3$, denote the gradient of the corresponding part with respect to the model parameters $\bm{\theta}$. For the OOD samples in $\mathcal{B}_{\text{wild}}$, evidently, they are correctly labeled (label-clean), the model parameters therefore will be updated in the correct gradient direction. 

As for ID samples, they consist of two parts: one part sampled from $\mathcal{D}_{\text{in}}^{\text{train}}$ (\textit{label-clean}), and the other part sampled from $\mathcal{D}_{\text{wild}}$ (\textit{label-noise}). Then the update of the model parameters $\bm{\theta}$ is as follows:
\begin{equation}
    \bm{\theta}^{t+1} = \bm{\theta}^{t} - \eta(\nabla\mathcal{L}_{N_1} + \nabla\mathcal{L}_{N_2}),
\end{equation}
where $t$ denotes the number of steps for model update, and $\eta$ is the learning rate. According to Eq.(6), the update of $\bm{\theta}$ is determined by $(\nabla\mathcal{L}_{N_1} + \nabla\mathcal{L}_{N_2})$. Since $|\mathcal{B}_{\text{in}}^{\text{train}}|>|\mathcal{B}_{\text{wild}}|$ and $|\mathcal{B}_{\text{wild}}|\geq N_2$, we have
$$|\mathcal{B}_{\text{in}}^{\text{train}}|>|\mathcal{B}_{\text{wild}}|\geq N_2.$$
This indicates that correctly labeled ID samples dominate the updating of model parameters, especially when $|\mathcal{B}_{\text{in}}^{\text{train}}|\gg N_2$. In summary, we have the labeled ID data $\mathcal{D}_{\text{in}}^{\text{train}}$ and the OOD data in $\mathcal{D}_{\text{wild}}$ that can jointly dominate the model learning, thus effectively addressing the model-bias issue.

\begin{table*}[]
\centering
\resizebox{0.885\textwidth}{!}{%
\begin{tabular}{@{}lccccccccccccc@{}}
\toprule
\multicolumn{1}{c}{\multirow{3}{*}{Methods}} & \multicolumn{12}{c}{OOD Dataset} & \multirow{3}{*}{ACC} \\ \cmidrule(lr){2-13}
\multicolumn{1}{c}{} & \multicolumn{2}{c}{SVHN} & \multicolumn{2}{c}{Places} & \multicolumn{2}{c}{LSUN-Crop} & \multicolumn{2}{c}{LSUN-Resize} & \multicolumn{2}{c}{Textures} & \multicolumn{2}{c}{Average} &  \\
\multicolumn{1}{c}{} & FPR95  & AUROC & FPR95 & AUROC & FPR95 & AUROC & FPR95 & AUROC & FPR95 & AUROC & FPR95 & AUROC &  \\ \midrule
\multicolumn{14}{c}{$\pi$=0.1} \\
OE (ICLR'19) & 1.57 & 99.63 & 60.24 & 83.43 & 3.83 & 99.26 & 0.93 & 99.79 & 27.89 & 93.35 & 18.89 & 95.09 & 71.65 \\
Energy(w/OE) (NeurIPS'20) & 1.47 & 99.68 & 54.67 & 86.09 & 2.52 & 99.44 & 2.68 & 99.50 & 37.26 & 91.26 & 19.72 & 95.19 & 73.46 \\
WOODS (ICML'22) & 0.12 & 99.96 & 29.58 & 90.60 & 0.11 & 99.96 & 0.07 & 99.96 & 9.12 & 96.65 & 7.80 & 97.43 & 75.22 \\
SAL (ICLR'24) & 0.07 & 99.95 & 3.53 & 99.06 & 0.06 & 99.94 & 0.02 & 99.95 & 5.73 & 98.65 & 1.88 & 99.51 & 73.71 \\
LoD (Ours) & \textbf{0} & \textbf{100} & \textbf{3.34} & \textbf{99.16} & \textbf{0} & \textbf{100} & \textbf{0} & \textbf{100} & \textbf{4.79} & \textbf{98.87} & \textbf{1.63} & \textbf{99.61} & 73.85 \\ \midrule
\multicolumn{14}{c}{$\pi$=0.5} \\
OE (ICLR'19)& 2.86 & 99.05 & 40.21 & 88.75 & 4.13 & 99.05 & 1.25 & 99.38 & 22.86 & 94.63 & 14.26 & 96.17 & 73.38 \\
Energy(w/OE) (NeurIPS'20) & 2.71 & 99.34 & 34.82 & 90.05 & 3.27 & 99.18 & 2.54 & 99.23 & 30.16 & 94.76 & 14.70 & 96.51 & 72.76 \\
WOODS (ICML'22) & 0.17 & 99.80 & 21.87 & 93.73 & 0.48 & 99.61 & 1.24 & 99.54 & 9.95 & 95.97 & 6.74 & 97.73 & 73.91 \\
SAL (ICLR'24) & 0.02 & 99.98 & \textbf{1.27} & 99.62 & 0.04 & 99.96 & 0.01 & 99.99 & 5.64 & 99.16 & 1.40 & 99.74 & 73.77 \\
LoD (Ours) & \textbf{0} & \textbf{100} & 1.53 & \textbf{99.66} & \textbf{0} & \textbf{100} & \textbf{0} & \textbf{100} & \textbf{3.72} & \textbf{99.19} & \textbf{1.05} & \textbf{99.77} & 74.32 \\ \midrule
\multicolumn{14}{c}{$\pi$=0.9} \\
OE (ICLR'19) & 0.84 & 99.36 & 19.78 & 96.29 & 1.64 & 99.57 & 0.51 & 99.75 & 12.74 & 94.95 & 7.10 & 97.98 & 72.02 \\
Energy(w/OE) (NeurIPS'20) & 0.97 & 99.64 & 17.52 & 96.53 & 1.36 & 99.73 & 0.94 & 99.59 & 14.01 & 95.73 & 6.96 & 98.24 & 73.62 \\
WOODS (ICML'22) & 0.05 & 99.98 & 11.34 & 95.83 & 0.07 & 99.99 & 0.03 & 99.99 & 6.72 & 98.73 & 3.64 & 98.90 & 73.86 \\
SAL (ICLR'24) & 0.03 & 99.99 & 2.79 & 99.89 & 0.05 & 99.99 & 0.01 & 99.99 & 5.88 & \textbf{99.53} & 1.75 & \textbf{99.88} & 74.01 \\
LoD (Ours) & \textbf{0} & \textbf{100} & \textbf{0.48} & \textbf{99.90} & \textbf{0} & \textbf{100} & \textbf{0} & \textbf{100} & \textbf{2.78} & 99.41 & \textbf{0.65} & 99.86 & 74.34 \\ \bottomrule
\end{tabular}%
}
\caption{Evaluation results of FPR95$\downarrow$ (\%), AUROC$\uparrow$ (\%) and ACC$\uparrow$ (\%) on standard benchmarks. CIFAR100 is ID, and bold numbers highlight the best results.}
\label{tab:main-ood}
\end{table*}

\begin{table*}[h]
\centering
\resizebox{0.885\textwidth}{!}{%
\begin{tabular}{@{}lcccccccclll@{}}
\toprule
\multicolumn{1}{c}{\multirow{3}{*}{Methods}} & \multicolumn{10}{c}{Dataset} & \multicolumn{1}{c}{\multirow{3}{*}{ACC}} \\ \cmidrule(lr){2-11}
\multicolumn{1}{c}{} & \multicolumn{2}{c}{CIFAR10} & \multicolumn{2}{c}{CIFAR+10} & \multicolumn{2}{c}{CIFAR+50} & \multicolumn{2}{c}{TinyImageNet} & \multicolumn{2}{c}{Average} & \multicolumn{1}{c}{} \\
\multicolumn{1}{c}{} & FPR95 & AUROC & FPR95 & AUROC & FPR95 & AUROC & FPR95 & AUROC & \multicolumn{1}{c}{FPR95} & \multicolumn{1}{c}{AUROC} & \multicolumn{1}{c}{} \\ \midrule
\multicolumn{12}{c}{$\pi=0.1$} \\
OE (ICLR'19) & 30.83 & 94.9 & 11.40 & 97.98 & 22.21 & 95.98 & 82.3 & 75.34 & 36.69 & 91.05 & 91.45 \\
Energy(w/OE) (NeurIPS'20) & 38.36 & 89.85 & 16.40 & 96.51 & 36.18 & 90.49 & 88.48 & 74.30 & 44.86 & 87.79 & 86.98 \\
WOODS (ICML'22) & 32.33 & 93.70 & 22.39 & 95.95 & 22.12 & 95.76 & 74.60 & 78.62 & 37.86 & 91.01 & 92.43 \\
SAL (ICLR'24) & 12.95 & 97.35 & 4.76 & 98.88 & 10.66 & 97.63 & 48.35 & 86.71 & 19.18 & 95.14 & 91.50 \\
LoD (Ours) & \textbf{2.56} & \textbf{99.40} & \textbf{1.50} & \textbf{99.62} & \textbf{1.96} & \textbf{99.39} & \textbf{47.61} & \textbf{91.55} & \textbf{13.41} & \textbf{97.49} & 91.44 \\ \midrule
\multicolumn{12}{c}{$\pi=0.5$} \\
OE (ICLR'19) & 13.77 & 97.68 & 4.08 & 99.09 & 9.80 & 98.27 & 76.13 & 80.62 & 25.95 & 93.92 & 91.82 \\
Energy(w/OE) (NeurIPS'20) & 9.16 & 97.70 & 3.70 & 98.98 & 10.01 & 97.43 & 75.93 & 83.58 & 24.70 & 94.42 & 87.91 \\
WOODS (ICML'22) & 17.89 & 96.64 & 12.50 & 97.69 & 12.68 & 97.68 & 70.60 & 81.42 & 28.42 & 93.36 & 92.53 \\
SAL (ICLR'24) & 12.76 & 97.38 & 4.84 & 98.87 & 10.86 & 97.60 & 48.17 & 86.77 & 19.16 & 95.16 & 91.39 \\
LoD (Ours) & \textbf{2.32} & \textbf{99.47}& \textbf{1.04} & \textbf{99.71} & \textbf{1.96} & \textbf{99.46} & \textbf{46.44} & \textbf{91.52} & \textbf{12.94} & \textbf{97.54} & 91.33 \\ \midrule
\multicolumn{12}{c}{$\pi=0.9$} \\
OE (ICLR'19) & 6.40 & 98.71 & 1.56 & 99.50 & 4.94 & 98.97 & 67.45 & 84.98 & 20.09 & 95.54 & 92.10 \\
Energy(w/OE) (NeurIPS'20) & 2.95 & 98.63 & 1.30 & 99.41 & 2.18 & 98.52 & 58.84 & 88.92 & 16.32 & 96.37 & 89.58 \\
WOODS (ICML'22) & 12.82 & 97.50 & 10.98 & 98.03 & 10.51 & 98.07 & 68.01 & 82.82 & 25.58 & 94.11 & 92.17 \\
SAL (ICLR'24) & 12.95 & 97.34 & 4.30 & 98.91 & 11.11 & 97.56 & 49.19 & 86.66 & 19.39 & 95.12 & 91.41 \\
LoD (Ours) & \textbf{2.19} & \textbf{99.45} & \textbf{1.04} & \textbf{99.77} & \textbf{1.90} & \textbf{99.45} & \textbf{45.24} & \textbf{91.80} & \textbf{12.59} & \textbf{97.62} & 91.50 \\ \bottomrule
\end{tabular}%
}
\caption{Evaluation results of FPR95$\downarrow$ (\%), AUROC$\uparrow$ (\%) and ACC$\uparrow$ (\%) on hard benchmarks, and bold numbers highlight the best results..}
\label{tab:main-osr}
\end{table*}
\subsection{Discriminability between ID and OOD CE Mean-Losses}
As mentioned earlier, the key to our LoD lies in ingeniously transforming the OOD filtering problem into a label-noise learning problem with controllable label-noise ratio, which allows us to leverage the established theoretical foundation of label-noise learning \cite{liu2020early,yue2024ctrl} to ensure the feasibility of our LoD. The work \cite{liu2020early} has shown that the phenomenon in early learning stage, when training with noisy labels, is intrinsic to high-dimensional classification tasks, even in the simplest setting, far from being a peculiar feature of deep neural networks. Therefore, for the second claim in Subsection 3.2, a theoretical analysis of loss gap between ID (label-noise) and OOD (label-clean) data in $\mathcal{D}_{\text{wild}}$ is provided here using a similar setting in \cite{liu2020early}. 

Considering a two class dataset that consists of $n$ independent samples $(\bm{x}_i,y_i)$ drawn from a mixture of two Gaussians in $\mathbb{R}^d$ as follows.
\begin{eqnarray*}
    &&\bm{x} \sim \mathcal{N}(+\bm{v}, \sigma^2\bm{I}_{d\times d}), \ \ \text{if} \ y = +1 \\
    &&\bm{x} \sim \mathcal{N}(-\bm{v}, \sigma^2\bm{I}_{d\times d}), \ \ \text{if} \ y = -1, 
\end{eqnarray*}
where $\bm{v}$ is an arbitrary unit vector in $\mathbb{R}^d$ and $\sigma^2$ is a small constant. Denote $y$ as the true hidden label and $\widetilde{y}$ as the observed label. Assume that for any sample $\bm{x}_i$,
\begin{equation}
    \widetilde{y} = \left\{
    \begin{array}{ll}
       y_i, \ &\text{with probability} \ 1 - \Delta,     \\
       -y_i, \ &\text{with probability} \ \Delta,  
    \end{array}
    \right.
\end{equation}
where $\Delta\in(0,1/2)$ is the label-noise ratio. Let us consider a linear classifier $f(\cdot,\bm{\theta})$ trained by gradient descent on CE loss:
\begin{equation}
    \min_{\bm{\theta\in\mathbb{R}^{2\times d}}} \mathcal{L}_{CE}(\bm{\theta}) := -\frac{1}{n}\sum_{i=1}^n\sum_{j=1}^2y_i\log(f(\bm{x_i},\bm{\theta})).
\end{equation}
In order to correctly classify the true classes well (and not overfit to the noisy labels), the rows of $\bm{\theta}$ should be correlated with the vector $\bm{v}$. Let $\nabla\mathcal{L}_{CE}(\bm{\theta})$ denote the gradient of Eq.(8). According to \cite{liu2020early}, we have the following lemma.
\begin{lemma}[Early-learning succeeds \cite{liu2020early}]
Denote by \{$\bm{\theta}_t$\} the iterates of gradient descent with step size $\eta$. For any $\Delta\in(0,1/2)$, there exists a constant $\delta_{\Delta}$, depending only on $\Delta$, such that if $\delta\leq \delta_{\Delta}$, then with high probability $1-o(1)$, there exists a $T=\Omega(1/\eta)$ such that: for all $t<T$, we have $\|\bm{\theta}_t-\bm{\theta}_0\|\leq1$ and  
$$-\nabla\mathcal{L}_{CE}(\bm{\theta}_t)^T\bm{v}/\|\nabla\mathcal{L}_{CE}(\bm{\theta}_t)\|\geq1/6.$$
\end{lemma}

Lemma 1 indicates that under the condition of label-noise ratio $\Delta$, the model parameters $\bm{\theta}$ update along the proper gradient direction during the early learning stage. This means, during this period, the loss curves of ID (\textit{label-noise}) and OOD (\textit{label-clean}) samples in $\mathcal{D}_{\text{wild}}$ will have significantly different characteristics, with larger loss values and greater fluctuations for ID samples versus smaller loss values and smaller fluctuations for OOD ones. To theoretically analyze this, we have the following proposition.

\begin{proposition}
Let $l_i$ denote the loss value of each sample in $\mathcal{D}_{\text{wild}}$, which is bounded by $R$. $\overline{l_{\text{in}}}=\frac{1}{|\mathcal{D}_{\text{in}}^{\text{wild}}|}\sum_{i\in \mathcal{D}_{\text{in}}^{\text{wild}}}l_i$ and $\overline{l_{\text{out}}}=\frac{1}{|\mathcal{D}_{\text{out}}^{\text{wild}}|}\sum_{i\in \mathcal{D}_{\text{out}}^{\text{wild}}} l_i$ respectively denote the mean losses of ID and OOD sets from unlabeled wild data $\mathcal{D}_{\text{wild}}$, and $n = |\mathcal{D}_{\text{in}}^{\text{wild}}| + |\mathcal{D}_{\text{out}}^{\text{wild}}|$. Under the Lemma 1, with high probability, we have
$$\overline{l_{\text{in}}} - \overline{l_{\text{out}}} \geq 1 - 2e^{-\bm{\theta}^T\bm{v}+\frac{1}{2}\|\bm{\theta}\|^2\delta^2} - \mathcal{O}(\frac{R}{\sqrt{n}}).$$
\end{proposition}
\newcommand{\myurl}{\url{https://github.com/ChuanxingGeng/LoD}}
Proposition 1 demonstrates that the cross-entropy mean losses of ID and OOD samples in $\mathcal{D}_{\text{wild}}$ are distinguishable, just as the two curves shown in Figure \ref{fig:fig3}. The proof is provided in Appendix A of supplementary materials (\myurl).


\section{Experiments}
\subsection{Implementation Details}
Our LoD (\myurl) framework contains two main modules, i.e., loss-difference OOD filtering module and OOD detector learning module. For these two modules, we follow \cite{du2024does,katz2022training} and employ Wide ResNet \cite{zagoruyko2016wide} with 40 layers and widen factor of 2 as the backbone. Moreover, for the loss-difference OOD filtering module, we use stochastic gradient descent with a momentum of 0.9 as the optimizer, and set the initial learning rate to 0.01. We train for 100 epochs using cosine learning rate decay, a batch size of 128 in which $|\mathcal{B}_{\text{in}}^{\text{train}}|:|\mathcal{B}_{\text{wild}}| = 3:1$ , and a dropout rate of 0.3. For the OOD detector learning module, similar to \cite{du2024does}, we load a pre-trained ID classifier and add an additional linear layer which utilize the penultimate-layer features of ID classifier for binary classification. The initial learning rate is set to 0.001, and the remaining training configurations are consistent with those of the former module. All experiments are conducted on a single NVIDIA RTX 3090 GPU. 

\paragraph{Evaluation Metrics.}
Similar to \cite{du2024does,katz2022training}, we adopt the following evaluation metrics: (1) the false positive rate (FPR95) of OOD examples when true positive rate of ID examples is at 95\%, (2) Area Under the Receiver Operating Characteristic curve (AUROC),  and (3) ID classification Accuracy (ACC). 

To comprehensively evaluate our LoD framework, we conduct extensive experiments on both standard benchmarks and hard benchmarks (newly curated in this paper) detailed in the following subsections. Moreover, limited by space, we defer additional experiments in the supplementary materials, including results on CIFAR10 (Appendix C), results on unseen OOD datasets (Appendix D), and results on different network structures (Appendix E).

\subsection{Experiments on Standard Benchmarks}
\paragraph{Datasets.}
For standard benchmarks, we here follow \cite{du2024does,katz2022training}, and choose CIFAR100 as in-distribution (ID) datasets ($\mathbb{P}_{\text{in}}$). For the out-of-distribution (OOD) test datasets  ($\mathbb{P}_{\text{out}}$), we use a diverse collection of natural image datasets including SVHN \cite{netzer2011reading}, Textures \cite{cimpoi2014describing}, Places \cite{zhou2017places}, LSUN-Crop \cite{yu2015lsun} and LSUN-Resize \cite{yu2015lsun}. For the unlabeled wild data ($\mathbb{P}_{\text{wild}}$), we follow \cite{du2024does} and mix datasets by combining a subset of ID data with OOD data under different mixture proportions $\pi \in \{0.1, 0.5, 0.9\}$. Specifically, the ID dataset is split into two equal halves (25,000 images per half), with one half used to mix with an OOD dataset (e.g., SVHN) to create the unlabeled wild data ($\mathbb{P}_{\text{wild}}$).

\paragraph{Main Results.}
We mainly compare our LoD with 4 latest methods using unlabeled wild data including Outlier Exposure (OE) \cite{hendrycks2018deep}, energy-regularization learning (Energy) \cite{liu2020energy}, WOODS \cite{katz2022training}, and SAL \cite{du2024does}. Table \ref{tab:main-ood} presents a comprehensive comparison of different methods on standard benchmarks, highlighting the substantial advantages of our proposed LoD. Across all datasets and $\pi$ values, our approach consistently delivers superior performance, achieving an FPR95 close to 0\%, which is significantly lower than the current SOTA baseline, SAL. Notably, on the most challenging Textures, our method outperforms SAL with substantial reductions in FPR95 by 0.94\%, 1.92\%, and 3.10\% for $\pi=0.1, 0.5, 0.9$, respectively. Moreover, while existing SOTA methods demonstrate strong performance in AUROC, our LoD achieves notable improvements even in this aspect. Importantly, our LoD maintains competitive in-distribution accuracy, matching or surpassing the performance of SOTA methods such as SAL and WOODS across various $\pi$ values.

\subsection{Experiments on Hard Benchmarks}
\paragraph{Datasets.}
In the settings of standard benchmarks, the ID and OOD samples are sourced from different datasets with inherently distinct distributions, which actually indirectly reduces the difficulty of OOD detection. As shown in Table 1, many methods, including ours, have achieved exceptionally high performance. To further demonstrate the advantages of our LoD, we here curate more challenging benchmarks, called hard benchmarks. Different from standard benchmarks, the ID and OOD samples on hard benchmarks come from the same dataset with different classes.


In specific, taking CIFAR10 as an example, we first randomly select 6 classes as ID data and the remaining 4 classes as OOD data. Then, similar to the splitting protocol of standard benchmarks, the training set of 6 ID classes is divided into two halves (15,000 images per half). One half is used as labeled ID data, while the other half is mixed with the data from 4 OOD classes to create the unlabeled wild data. We here select CIFAR10,  CIFAR+10, CIFAR+50, and TinyImageNet \cite{Vaze2022OpenSetRA} to curate the hard OOD benchmarks, and more details can be found in Appendix B of supplementary materials.

\paragraph{Main Results.}
Since the four methods we compared do not conduct the experiments on these benchmarks, we reproduce the results according to the source codes provided by them. 
Table \ref{tab:main-osr} reports the detailed results on hard benchmarks. Across all datasets and under various $\pi$ values, our LoD achieves better FPR95 and AUROC performance compared to existing methods, indicating that its OOD detection has stronger generalization. Notably, compared to the SOTA baseline SAL \cite{du2024does}, our method reduces FPR95 by substantial margins of 5.77\%, 6.22\%, and 6.80\% on average when $\pi=0.1, 0.5, 0.9$, respectively. Especially on CIFAR10, where LoD outperforms SAL more than 10\% in case of FPR95. In particular, on the most challenging TinyImageNet, LoD consistently surpasses SAL by a large margin of 4.84\%, 4.75\%, and 5.14\% in terms of AUROC when $\pi=0.1, 0.5, 0.9$, respectively. Besides, our LoD also maintains competitive ID classification accuracy compared to the SOTA baseline, comprehensively demonstrating the effectiveness of our LoD.
\begin{figure}[h]
  \centering
  \subfigure{\includegraphics[width=4.2cm, height=3cm]{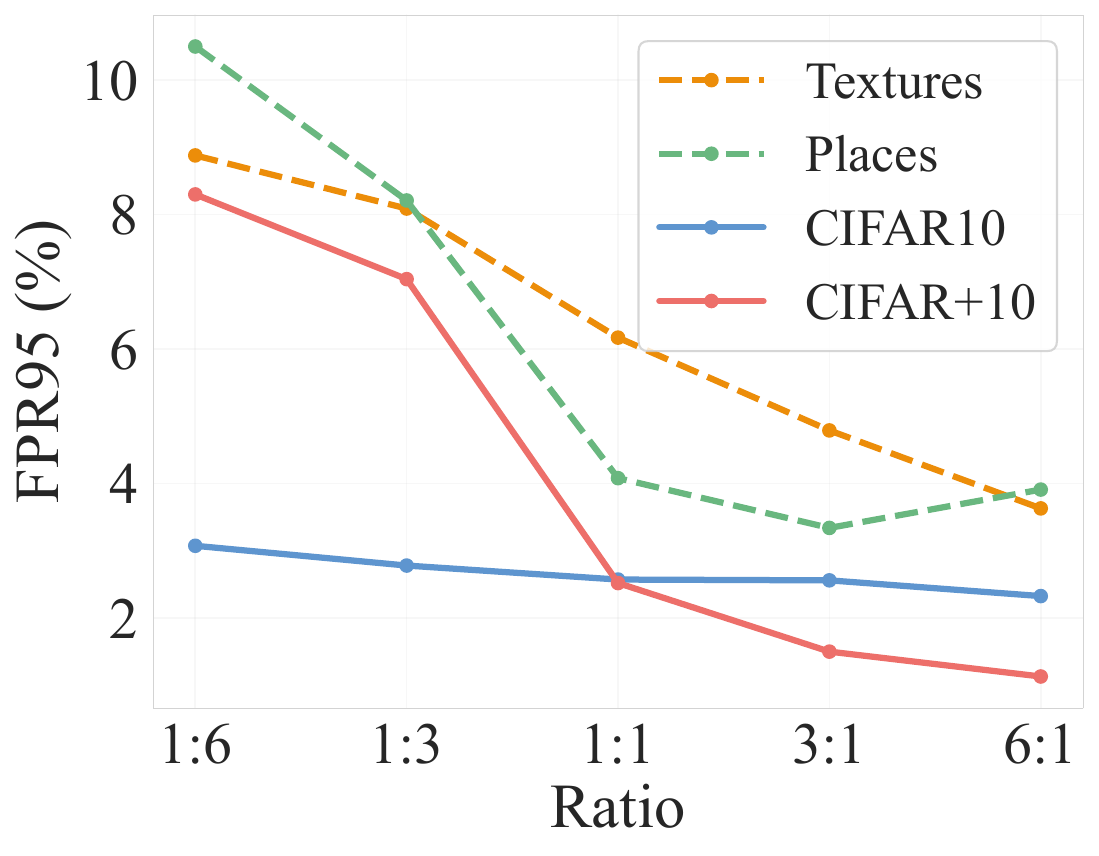}}
  \hfill
  \subfigure{\includegraphics[width=4.2cm, height=3cm]{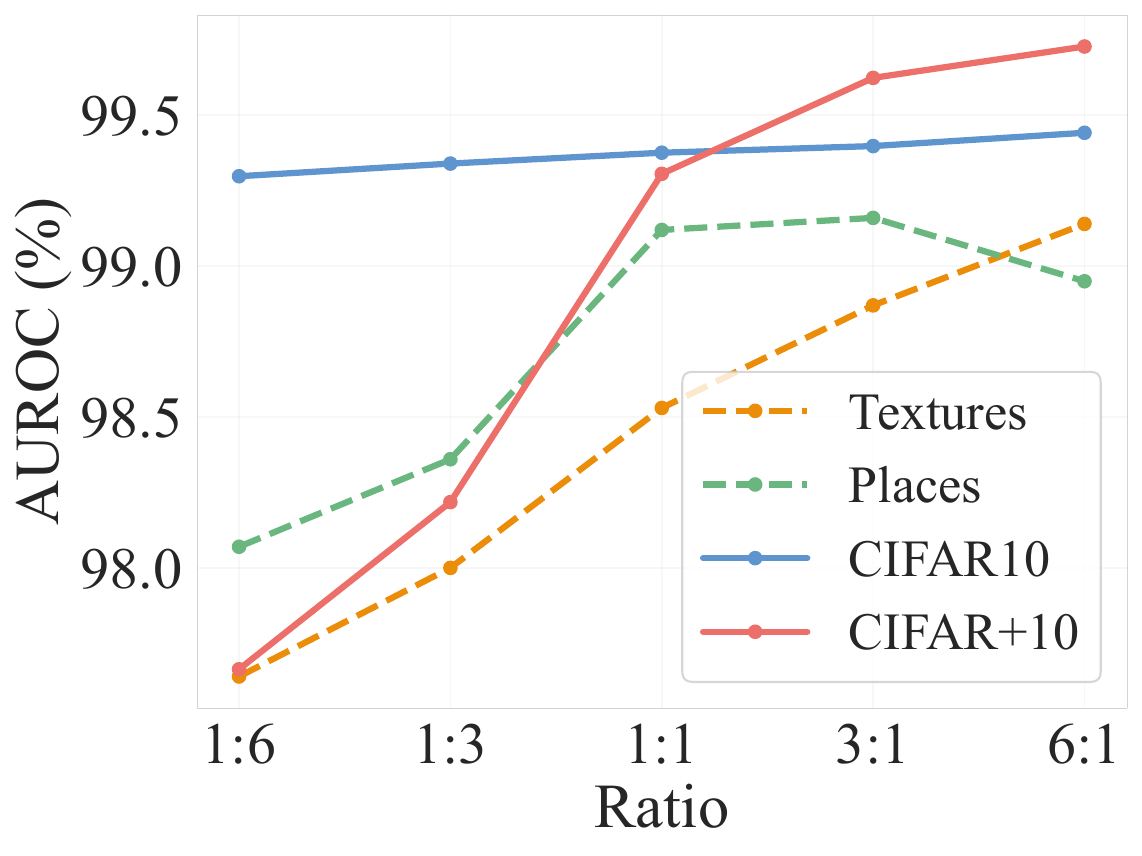}}
  
  \caption{Experiments in different rations ($|\mathcal{B}_{\text{in}}^{\text{train}}|/|\mathcal{B}_{\text{wild}}|$) on standard benchmarks (dashed lines) and hard benchmarks (solid lines). 
  }
  \label{fig:ratio}
\end{figure}


\subsection{Experiments on Different Ratios and Epochs}
\subsubsection{Results on Different Ratios of $|\mathcal{B}_{\text{in}}^{\text{train}}|/|\mathcal{B}_{\text{wild}}|$}
According to Section 4.1, the larger the ratio $|\mathcal{B}_{\text{in}}^{\text{train}}|/|\mathcal{B}_{\text{wild}}|$, the more dominant the labeled ID data in $\mathcal{D}_{\text{in}}^{\text{train}}$ and the OOD data in $\mathcal{D}_{\text{wild}}$ are in model learning, thus leading to better model performance. To verify this, we conduct experiments in different ratios of $|\mathcal{B}_{\text{in}}^{\text{train}}|/|\mathcal{B}_{\text{wild}}|$. Figure \ref{fig:ratio} shows the results. As the ratio increases, the model performance consistently improves across all benchmarks, strongly supporting our claim. Considering computational efficiency, $|\mathcal{B}_{\text{in}}^{\text{train}}|/|\mathcal{B}_{\text{wild}}|$ is set to $3:1$ in all of our experiments.

\begin{figure}[h]
  \centering
  \subfigure{\includegraphics[width=4.2cm, height=3cm]{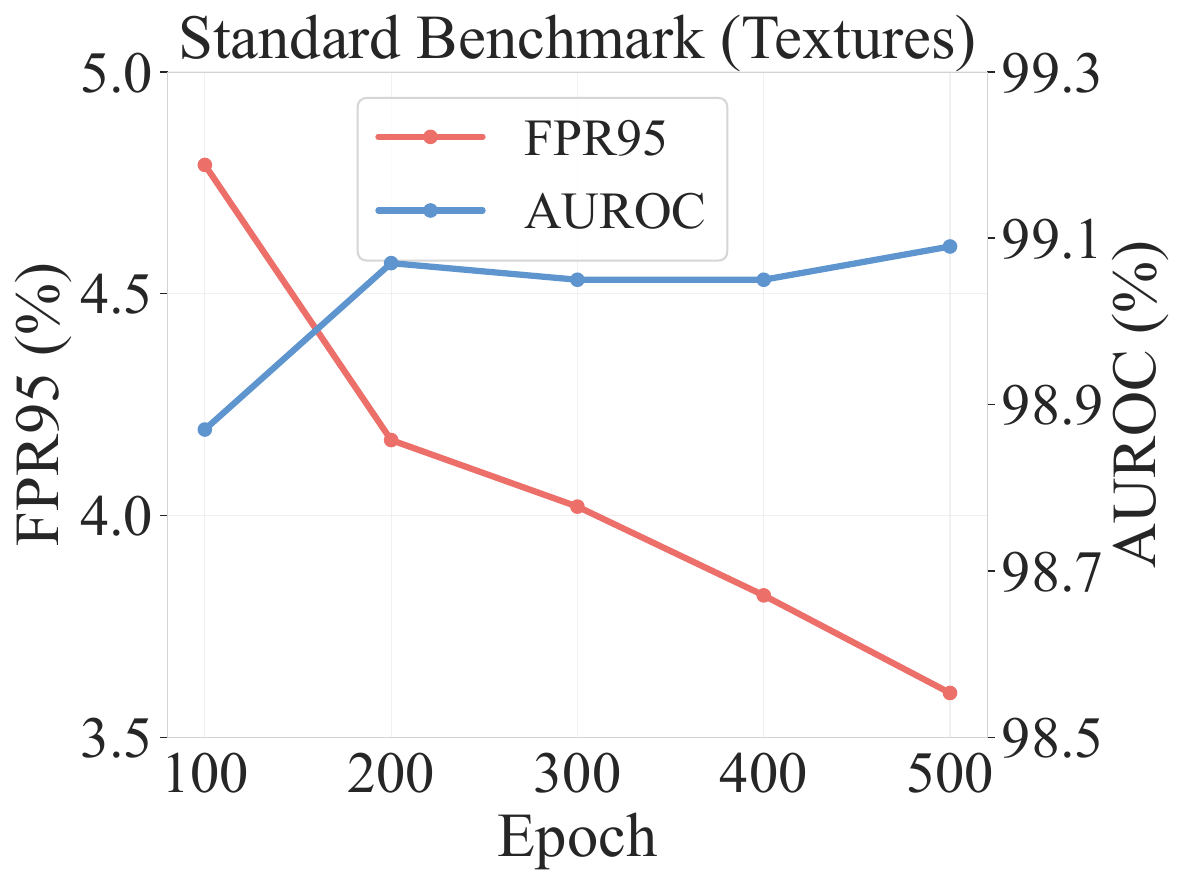}}
  \hfill
  \subfigure{\includegraphics[width=4.2cm, height=3cm]{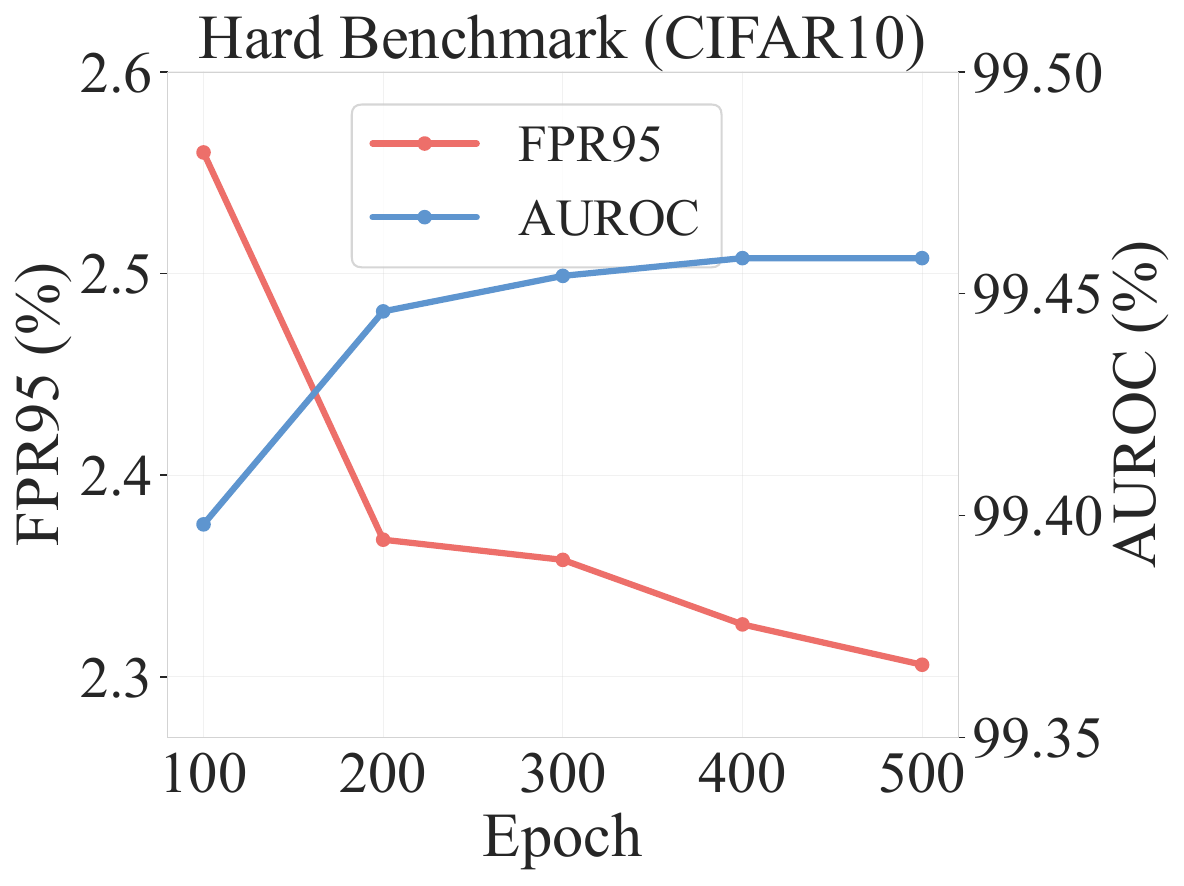}}
  
  \caption{The impacts of training epochs on results respectively
in standard and hard benchmarks.
  }
  \label{fig:epoch}
\end{figure}

\subsubsection{Impact of Epoch in Early-learning Succeeds}
As shown in Proposition 1, the early-learning succeeds is the key to our LoD. To clearly demonstrate the appropriate number of training epochs, we conduct the epoch experiments on standard benchmark (take Textures as an example) and hard benchmark (take CIFAR10 as an example) respectively. Figure \ref{fig:epoch} shows the results, and we can observe a steady performance improvement in our LoD from 100 to 500 training epochs. At first glance, this phenomenon seems inconsistent with the early-learning succeeds in the traditional label-noise learning field, which is usually shorter. However, please note that in our work setting, the label-noise ratio is controlled within an appropriate range by controlling the ratio of $|\mathcal{B}_{\text{in}}^{\text{train}}|/|\mathcal{B}_{\text{wild}}|$, meaning that correctly labeled samples all along dominate the network's learning. This further verifies the operability of LoD due to the long period early-learning succeeds. Considering efficiency issues, the training epochs of all experiments in this paper are set to 100 epochs.

\section{Conclusion}
In this paper, we innovatively propose a loss-difference OOD detection framework by \textit{intentionally label-noisifying} unlabeled wild data, which ingeniously transforms the OOD filtering problem in unlabeled wild data into a label-noise learning problem with controllable label-noise ratio. Importantly, LoD not only effectively addresses the model-bias issue commonly associated with existing methods, but also circumvents the threshold selection dilemma inherent in these approaches. 

\section*{Acknowledgments}
This research was supported in part by the National Natural Science Foundation of China (62106102, 62376126, 62272229), in part by the NSFC-Hong Kong joint collaboration research fund CRS\_HKU703/24, in part by the Hong Kong Scholars Program under Grant XJ2023035, in part by the Fundamental Research Funds for the Central Universities under Grant NS2024058.

\bibliographystyle{named}
\bibliography{ijcai25}

\appendix
\section{Detailed Proof}
\label{app:proofs}
\setcounter{lemma}{0}
\setcounter{proposition}{0}
To prove Proposition 1, we first reintroduce Lemma 1 from \cite{liu2020early} and Proposition 1 as follows:
\begin{lemma}[Early-learning succeeds]
Denote by \{$\bm{\theta}_t$\} the iterates of gradient descent with step size $\eta$. For any $\Delta\in(0,1/2)$, there exists a constant $\delta_{\Delta}$, depending only on $\Delta$, such that if $\delta\leq \delta_{\Delta}$, then with high probability $1-o(1)$, there exists a $T=\Omega(1/\eta)$ such that: for all $t<T$, we have $\|\bm{\theta}_t-\bm{\theta}_0\|\leq1$ and  
$$-\nabla\mathcal{L}_{CE}(\bm{\theta}_t)^T\bm{v}/\|\nabla\mathcal{L}_{CE}(\bm{\theta}_t)\|\geq1/6.$$
\end{lemma}

\begin{proposition}
Let $l_i$ denote the loss value of each sample in $\mathcal{D}_{\text{wild}}$, which is bounded by $R$. $\overline{l_{\text{in}}}=\frac{1}{|\mathcal{D}_{\text{in}}^{\text{wild}}|}\sum_{i\in \mathcal{D}_{\text{in}}^{\text{wild}}}l_i$ and $\overline{l_{\text{out}}}=\frac{1}{|\mathcal{D}_{\text{out}}^{\text{wild}}|}\sum_{i\in \mathcal{D}_{\text{out}}^{\text{wild}}} l_i$ respectively denote the mean losses of ID and OOD sets from unlabeled wild data $\mathcal{D}_{\text{wild}}$, and $n = |\mathcal{D}_{\text{in}}^{\text{wild}}| + |\mathcal{D}_{\text{out}}^{\text{wild}}|$. Under the Lemma 1, with high probability, we have
$$\overline{l_{\text{in}}} - \overline{l_{\text{out}}} \geq 1 - 2e^{-\bm{\theta}^T\bm{v}+\frac{1}{2}\|\bm{\theta}\|^2\delta^2} - \mathcal{O}(\frac{R}{\sqrt{n}}).$$
\end{proposition}
\begin{proof}
    Lemma 1 indicates that under the condition of noise level $\Delta$, the model parameters $\bm{\theta}$ update along the proper gradient direction during the early learning stage. This means, during this period, the loss curves of ID (\textit{label-noise}) and OOD (\textit{label-clean}) samples in test-set will have significantly different characteristics, with larger loss values and greater fluctuations for ID samples versus smaller loss values and smaller fluctuations for OOD ones. Next, we analyze the mean loss gap between ID (label-noise) samples in $\mathcal{D}_{\text{in}}^{\text{wild}}$ and OOD (label-clean) samples in $\mathcal{D}_{\text{out}}^{\text{wild}}$ during this stage. Following \cite{yue2024ctrl}, we adopt sigmoid function as the activation function for the network outputs. For each sample $(\bm{x_i},y_i)$, we have
\begin{eqnarray*}
    &&p(y_i=1) = \text{sig}(\bm{\theta}^T\bm{x_i}) = \frac{1}{1+e^{-\bm{\theta}^T\bm{x_i}}}, \\
    &&p(y_i=-1) = 1 - p(y_i=1).
\end{eqnarray*}
Let $\bm{x} = \bm{v} + \bm{z}_i$, where $\bm{z}_i\sim\mathcal{N}(\bm{0},\sigma^2\bm{I}_{d\times d})$. For each sample $\bm{x_i}\in \mathcal{D}_{\text{out}}^{\text{wild}}$ (label-clean), we use $\log$ for its loss, and have
\begin{equation*}
    l_i(\bm{\theta}) = \log(1+e^{-\bm{\theta}^T(\bm{v}+\bm{z}_i)}) \leq e^{-\bm{\theta}^T(\bm{v}+\bm{z}_i)}.
\end{equation*}
Similarly, for each sample $\bm{x_j}\in \mathcal{D}_{\text{in}}^{\text{wild}}$ (label-noise), we have
\begin{equation*}
    l_j(\bm{\theta}) = \log(1+e^{\bm{\theta}^T(\bm{v}+\bm{z}_i)}) \geq 1 - e^{-\bm{\theta}^T(\bm{v}+\bm{z}_i)}.
\end{equation*}
Taking the expectation on the difference between OOD (label-clean) and ID (label-noise), we have
\begin{equation*}
    \mathbb{E}[l_i(\bm{\theta}) - l_i(\bm{\theta})] = \mathbb{E}[l_i(\bm{\theta})] - \mathbb{E}[l_j(\bm{\theta})] \geq 1 - 2\cdot\mathbb{E}[e^{-\bm{\theta}^T(\bm{v}+\bm{z})}].
\end{equation*}
Note that the term $1 - 2\cdot\mathbb{E}[e^{-\bm{\theta}^T(\bm{v}+\bm{z})}]$ bounds the loss gap between OOD (label-clean) and know-class (label-noise) samples, and it is independent of the label type. Since
\begin{equation}
    \mathbb{E}[e^{-\bm{\theta}^T(\bm{v}+\bm{z})}] = e^{-\bm{\theta}^T\bm{v}}\cdot\mathbb{E}[e^{-\bm{\theta}^T\bm{z}}] = e^{-\bm{\theta}^T\bm{v}}\cdot e^{\frac{1}{2}\|\bm{\theta}\|^2\sigma^2}.
\end{equation}
Eq.(1) indicates that the smaller the $\sigma$ or the projection $\bm{\theta}$ has on $\bm{v}$, the larger the expected loss gap. Interestingly, Lemma 1 ensures that we can obtain a good $\bm{\theta}$ at least within $T$ epochs. Define the mean losses of ID (label-noise) samples  and OOD (label-clean) samples as follows:
\begin{equation*}
    \overline{l_{\text{in}}}=\frac{1}{|\mathcal{D}_{\text{in}}^{\text{wild}}|}\sum_{i\in \mathcal{D}_{\text{in}}^{\text{wild}}}l_i, \ \ \ \overline{l_{\text{out}}}=\frac{1}{|\mathcal{D}_{\text{out}}^{\text{wild}}|}\sum_{i\in \mathcal{D}_{\text{out}}^{\text{wild}}}l_i.
\end{equation*}
By Hoeffding’s Inequality on bounded variables and the Union Bound, with probability $\geq 1 - \delta$, we have
\begin{equation}
    \overline{l_{\text{in}}} \geq \mathbb{E}[\overline{l_{\text{in}}}] - \mathcal{O}(\frac{R}{\sqrt{|\mathcal{D}_{\text{in}}^{\text{wild}}|}}\sqrt{\log\frac{1}{\delta}}).
\end{equation}
and
\begin{equation}
    \overline{l_{\text{out}}} \leq \mathbb{E}[\overline{l_{\text{out}}}] + \mathcal{O}(\frac{R}{\sqrt{|\mathcal{D}_{\text{out}}^{\text{wild}}|}}\sqrt{\log\frac{1}{\delta}}).
\end{equation}
According to Eq.(2) and Eq.(3), we have
$$\overline{l_{\text{in}}} - \overline{l_{\text{out}}} \geq 1 - 2e^{-\bm{\theta}^T\bm{v}+\frac{1}{2}\|\bm{\theta}\|^2\delta^2} - \mathcal{O}(\frac{R}{\sqrt{n}}).$$

\end{proof}

\begin{table*}[h]
\centering
\resizebox{\textwidth}{!}{%
\begin{tabular}{@{}lccccccccccccc@{}}
\toprule
\multicolumn{1}{c}{\multirow{3}{*}{Methods}} & \multicolumn{12}{c}{OOD Dataset} & \multirow{3}{*}{ACC} \\ \cmidrule(lr){2-13}
\multicolumn{1}{c}{} & \multicolumn{2}{c}{SVHN} & \multicolumn{2}{c}{Places} & \multicolumn{2}{c}{LSUN-Crop} & \multicolumn{2}{c}{LSUN-Resize} & \multicolumn{2}{c}{Textures} & \multicolumn{2}{c}{Average} &  \\
\multicolumn{1}{c}{} & FPR95& AUROC& FPR95& AUROC& FPR95& AUROC& FPR95& AUROC& FPR95& AUROC& FPR95& AUROC&  \\ \midrule

\multicolumn{14}{c}{With $\mathbb{P}_{\text{in}}$ only} \\

MSP (ICLR'17) & 48.49 & 91.89 & 59.48 & 88.20 & 30.80 & 95.65 & 52.15 & 91.37 & 59.28 & 88.50 & 50.04 & 91.12 & 94.84 \\
ODIN (ICLR'18)  & 33.35 & 91.96 & 57.40 & 84.49 & 15.52 & 97.04 & 26.62 & 94.57 & 49.12 & 84.97 & 36.40 & 90.61 & 94.84 \\
Mahalanobis (NeurIPS'18) & 12.89 & 97.62 & 68.57 & 84.61 & 39.22 & 94.15 & 42.62 & 93.23 & 15.00 & 97.33 & 35.66 & 93.34 & 94.84 \\
Energy (NeurIPS'20) & 35.59 & 90.96 & 40.14 & 89.89 & 8.26 & 98.35 & 27.58 & 94.24 & 52.79 & 85.22 & 32.87 & 91.73 & 94.84 \\
CSI (NeurIPS'20) & 17.30 & 97.40 & 34.95 & 93.64 & 1.95 & 99.55 & 12.15 & 98.01 & 20.45 & 95.93 & 17.36 & 96.91 & 94.17 \\
ReAct (NeurIPS'21) & 40.76 & 89.57 & 41.44 & 90.44 & 14.38 & 97.21 & 33.63 & 93.58 & 53.63 & 86.59 & 36.77 & 91.48 & 94.84 \\
KNN (ICML'22) & 24.53 & 95.96 & 25.29 & 95.69 & 25.55 & 95.26 & 27.57 & 94.71 & 50.90 & 89.14 & 30.77 & 94.15 & 94.84 \\
KNN+ (ICML'22) & 2.99 & 99.41 & 24.69 & 94.84 & 2.95 & 99.39 & 11.22 & 97.98 & 9.65 & 98.37 & 10.30 & 97.99 & 93.19 \\
DICE (ECCV'22) & 35.44 & 89.65 & 46.83 & 86.69 & 6.32 & 98.68 & 28.93 & 93.56 & 53.62 & 82.20 & 34.23 & 90.16 & 94.84 \\
ASH (ICLR'23) & 6.51 & 98.65 & 48.45 & 88.34 & 0.90 & 99.73 & 4.96 & 98.92 & 24.34 & 95.09 & 17.03 & 96.15 & 94.84 \\ \midrule

\multicolumn{14}{c}{With $\mathbb{P}_{\text{in}}$ and $\mathbb{P}_{\text{wild}}$} \\
OE (ICLR'19) & 0.85 & 99.82 & 23.47 & 94.62 & 1.84 & 99.65 & 0.33 & 99.93 & 10.42 & 98.01 & 7.38 & 98.41 & 94.07 \\
Energy(w/OE) (NeurIPS'20) & 4.95 & 98.92 & 17.26 & 95.84 & 1.93 & 99.49 & 5.04 & 98.83 & 13.43 & 96.69 & 8.52 & 97.95 & 94.81 \\
WOODS (ICML'22) & 0.15 & 99.97 & 12.49 & 97.00 & 0.22 & 99.94 & 0.03 & 99.99 & 5.95 & 98.79 & 3.77 & 99.14 & 94.84 \\
SAL (ICLR'24) & 0.02 & 99.98 & 2.57 & 99.24 & 0.07 & 99.99 & 0.01 & 99.99 & 0.90 & 99.74 & 0.71 & 99.78 & 93.65 \\
LoD (Ours) & \textbf{0} & \textbf{100} & \textbf{1.72} & \textbf{99.52} & \textbf{0} & \textbf{100} & \textbf{0} & \textbf{100} & \textbf{0.66} & \textbf{99.90} & \textbf{0.48} & \textbf{99.88} & 94.06 \\ \bottomrule
\end{tabular}%
}
\caption{Evaluation results of FPR95$\downarrow$ (\%), AUROC$\uparrow$ (\%) and ACC$\uparrow$ (\%) on standard benchmarks. CIFAR10 is ID dataset, and bold numbers highlight the best results.}
\label{tab:ood_cifar10}
\end{table*}

\begin{table*}[h]
\centering
\resizebox{\textwidth}{!}{%
\begin{tabular}{@{}lccccccccccccccc@{}}
\toprule
\multicolumn{1}{c}{\multirow{3}{*}{Methods}} & \multicolumn{14}{c}{OOD Dataset} & \multirow{3}{*}{ACC} \\ \cmidrule(lr){2-15}
\multicolumn{1}{c}{} & \multicolumn{2}{c}{SVHN} & \multicolumn{2}{c}{Places} & \multicolumn{2}{c}{LSUN-Crop} & \multicolumn{2}{c}{LSUN-Resize} & \multicolumn{2}{c}{Textures} & \multicolumn{2}{c}{25K RAND.IMG.} & \multicolumn{2}{c}{Average} &  \\
\multicolumn{1}{c}{} & FPR95& AUROC& FPR95& AUROC& FPR95& AUROC& FPR95& AUROC& FPR95& AUROC& FPR95& AUROC& FPR95& AUROC&  \\ \midrule
\multicolumn{16}{c}{$\pi$=0.1} \\
OE (ICLR'19) & 77.74 & 82.84 & 60.70 & 84.02 & 31.06 & 93.99 & 55.74 & 88.45 & 57.39 & 88.27 & 50.95 & 87.44 & 56.53 & 87.51 & 83.04 \\
Energy(w/OE) (NeurIPS'20) & 55.89 & 90.19 & 49.08 & 88.03 & 22.74 & 94.94 & 34.10 & 93.42 & 39.33 & 90.63 & 48.91 & 88.12 & 40.23 & 91.44 & 90.02 \\
WOODS (ICML'22) & 4.90 & 98.70 & 18.53 & 96.27 & 1.94 & 99.53 & 5.73 & \textbf{98.78} & 17.71 & 96.17 & 10.37 & 96.92 & 9.76 & 97.89 & 94.50 \\
SAL (ICLR'24) & 5.83 & 97.63 & 17.96 & 96.23 & 2.50 & 98.77 & 5.67 & 98.56 & 8.44 & 97.93 & 8.95 & 97.40 & 8.08 & 97.82 & 93.65 \\
LoD (Ours) & \textbf{4.41} & \textbf{98.96} & \textbf{11.82} & \textbf{97.50} & \textbf{1.84} & \textbf{99.56} & \textbf{5.61} & 98.63 & \textbf{4.66} & \textbf{99.10} & \textbf{8.68} & \textbf{97.59} & \textbf{5.67} & \textbf{98.75} & 93.99 \\ \bottomrule
\end{tabular}%
}
\caption{Evaluation results of FPR95$\downarrow$ (\%), AUROC$\uparrow$ (\%) and ACC$\uparrow$ (\%) on unseen datasets. We use CIFAR10 as ID and a subset (25K images) of 300K Random Images as wild OOD data. Bold numbers
highlight the best results}
\label{tab:diff_cifar10}
\end{table*}

\section{Details in Hard Benchmarks}
To further demonstrate the advantages of our LoD, we conduct experiments on curated hard OOD benchmarks including CIFAR10, CIFAR+10, CIFAR+50, and TinyImageNet. The details of these benchmarks are as follows:
\begin{itemize}
    \item \textbf{CIFAR10.} CIFAR10 \cite{Krizhevsky2009} contains 10 classes, where 6 classes are randomly selected as in-distribution (ID) classes, and the remaining 4 classes are used as out-of-distribution (OOD) classes.
    
    \item \textbf{CIFAR+10 \& CIFAR+50.} In this set of experiments, 4 classes from CIFAR10 are randomly selected as ID classes, and 10/50 non-overlapping classes randomly selected from CIFAR100 \cite{Krizhevsky2009} are OOD classes.
    
    \item \textbf{TinyImageNet.} TinyImageNet is a subset derived from ImageNet \cite{deng2009imagenet} with a total of 200 classes, of which 20 classes are randomly selected as ID classes and the rest 180 classes are treated as OOD classes.
\end{itemize}

Please note that, since ID and OOD are randomly divided, to mitigate the effects of randomness, each dataset is evaluated across five distinct "ID/OOD" splits following \cite{neal2018open,Vaze2022OpenSetRA}, and the results are averaged. Moreover, similar to standard benchmarks \cite{katz2022training,du2024does}, we use 70\% of data from the OOD classes as the OOD part of the unlabeled wild data.


\section{Additional Results on CIFAR10}
In this part, we utilize CIFAR10 as the ID dataset to evaluate our LoD under $\pi=0.1$. In addition to the four methods utilizing wild data compared in the main paper, we here also evaluate methods that rely solely on labeled ID data ($\mathbb{P}_\text{in}$ only) including MSP \cite{hendrycks2016baseline}, ODIN \cite{liang2017enhancing}, Mahalanobis \cite{lee2018simple}, Energy \cite{liu2020energy}, CSI \cite{tack2020csi}, ReAct \cite{sun2021react}, KNN and KNN+ \cite{sun2022out}, DICE \cite{sun2022dice} and ASH \cite{djurisic2022extremely}.
The detailed results are presented in Table \ref{tab:ood_cifar10}, which demonstrate that methods trained using both ID and wild data exhibit significantly better performance compared to those trained solely with ID data. Additionally, compared with methods utilizing $\mathbb{P}_\text{wild}$, LoD continues to exhibit superior performance, outperforming other SOTA methods in terms of FPR95 and AUROC metrics. Furthermore, LoD achieves competitive ID classification accuracy, either matching or exceeding the performance of leading SOTA methods such as SAL and WOODS.

\section{Additional Results on Unseen OOD Datasets}
In this part, we follow \cite{du2024does} and evaluate our LoD on unseen OOD datasets, which are different from the OOD data we use in the wild. Table 2 and Table 3 report the results.

In Table 2, we employ CIFAR10 as the ID dataset. As for the wild OOD data, \cite{du2024does} utilizes the full 300K-image dataset. However, we argue that this setting seems inappropriate due to a significant imbalance: the ID data in the wild data contains only 25K images, while the OOD counterpart comprises 300K images--12 times larger than the ID data. Therefore, we randomly sample a subset of 25K images from the 300K as the wild OOD data. The detailed results presented in Table \ref{tab:diff_cifar10} demonstrate that our LoD consistently outperforms SOTA baselines such as SAL and WOODS on the unseen OOD datasets, highlighting the effectiveness of our method.

In Table 3, we employ CIFAR100 as ID data. As for the wild OOD data, we follow \cite{du2024does} and utilize TinyImageNet-crop (TINc)/TinyImageNet-resize (TINr) dataset as the wild OOD data using during training and TINr/TINc as the unseen OOD data during testing. The results in Table \ref{tab:diff_c100} demonstrate the advantages of our LoD.

\begin{table}[h]
\centering
\tabcolsep 2mm
\footnotesize
\begin{tabular}{@{}lcccc@{}}
\toprule
\multicolumn{1}{c}{\multirow{3}{*}{Methods}} & \multicolumn{4}{c}{OOD Dataset} \\ \cmidrule(l){2-5} 
\multicolumn{1}{c}{} & \multicolumn{2}{c}{TINr} & \multicolumn{2}{c}{TINc} \\
\multicolumn{1}{c}{} & FPR95& AUROC& FPR95& AUROC\\ \midrule
STEP (NeurIPS'21) & 72.31 & 74.59 & 48.68 & 91.14 \\
TSL (MM'23) & 57.52 & 82.29 & 29.48 & 94.62 \\
SAL (ICLR'24) & 43.11 & 89.17 & 19.30 & 96.29 \\
LoD (Ours) & \textbf{23.54} & \textbf{92.81} & \textbf{9.67} & \textbf{98.10} \\ \bottomrule
\end{tabular}%
\caption{Evaluation results of FPR95$\downarrow$ (\%), AUROC$\uparrow$ (\%) on unseen datasets. CIFAR100 is ID, and bold numbers
highlight the best results.}
\label{tab:diff_c100}
\end{table}


\section{Additional Results on Different Networks}
To verify the applicability of LoD, the data-centric method, this part conducts experiments on different network structures on
CIFAR10 and CIFAR+10. Table 4 reports the results, and we
can find these networks mentioned here are all suitable for
our LoD. In particular, LoD seems to follow scaling laws:
the larger the network, the better it performs.
\begin{table}[h]
\centering
\resizebox{\columnwidth}{!}{%
\begin{tabular}{@{}lcccccc@{}}
\toprule
\multirow{2}{*}{Networks\scriptsize($\sharp$ params)} & \multicolumn{3}{c}{CIFAR10} & \multicolumn{3}{c}{CIFAR+10} \\
                        & FPR95  & AUROC   & ACC     & FPR95    & AUROC     & ACC      \\ \midrule
WideResNet-40-2 {\scriptsize(2.2M)}   &  2.56   &  99.40  &  96.34  &  1.50    &  99.62 & 97.29   \\
ResNet18 {\scriptsize(11.2M)}          & 2.47    &  99.44  &  96.46  &  0.96    &  99.72 & 97.32   \\
ResNet34 {\scriptsize(21.3M)}          & 2.29    &  99.51  &  96.48  &  0.90    &  99.73 & 97.39   \\ \bottomrule
\end{tabular}%
}
\caption{Evaluation results of FPR95$\downarrow$ (\%), AUROC$\uparrow$ (\%) and ACC$\uparrow$ (\%) on different networks on hard benchmarks.}
\label{tab:backbone}
\end{table}

\begin{table}[h]
\centering
\resizebox{\columnwidth}{!}{%
\begin{tabular}{@{}ccccccc@{}}
\toprule
\multirow{2}{*}{Ratios} & \multicolumn{3}{c}{Places} & \multicolumn{3}{c}{Textures} \\
                        & FPR95 & AUROC   & ACC    & FPR95   & AUROC  & ACC  \\ \midrule
1:6                     & 10.50    & 98.07    & 72.9    & 8.88     & 97.64   & 74.10   \\
1:3                     & 8.21     & 98.36    & 73.14   & 8.09     & 98.00   & 73.58   \\
1:1                     & 4.08     & 99.12    & 73.06   & 6.17     & 98.53   & 74.06   \\
3:1                     & 3.34     & 99.16    & 72.21   & 4.79     & 98.87   & 73.30   \\
6:1                     & 3.91     & 98.95    & 71.38   & 3.63     & 99.14   & 73.22   \\ \bottomrule
\end{tabular}%
}
\caption{Detailed results of FPR95$\downarrow$ (\%), AUROC$\uparrow$ (\%) and ACC$\uparrow$ (\%) across different ratios on standard benchmarks.}
\label{tab:details_ratios_std}
\end{table}

\begin{table}[h]
\centering
\resizebox{\columnwidth}{!}{%
\begin{tabular}{@{}ccccccc@{}}
\toprule
\multirow{2}{*}{Ratios} & \multicolumn{3}{c}{CIFAR10} & \multicolumn{3}{c}{CIFAR+10} \\
                        & FPR95  & AUROC   & ACC     & FPR95    & AUROC     & ACC      \\ \midrule
1:6                     & 3.07    & 99.30   & 96.18   & 8.30     & 97.66     & 97.35   \\
1:3                     & 2.78    & 99.34   & 96.27   & 7.04     & 98.22     & 97.25   \\
1:1                     & 2.57    & 99.38   & 96.28   & 2.52     & 99.31     & 97.25   \\
3:1                     & 2.56    & 99.40   & 96.34   & 1.50     & 99.62     & 97.29   \\
6:1                     & 2.33    & 99.44   & 96.21   & 1.13     & 99.73     & 97.34    \\ \bottomrule
\end{tabular}%
}
\caption{Detailed results of FPR95$\downarrow$ (\%), AUROC$\uparrow$ (\%) and ACC$\uparrow$ (\%) across different ratios on hard benchmarks.}
\label{tab:details_ratios_hard}
\end{table}

\begin{table}[h]
\centering
\resizebox{\columnwidth}{!}{%
\begin{tabular}{@{}ccccccc@{}}
\toprule
\multirow{2}{*}{Ratios} & \multicolumn{3}{c}{Textures} & \multicolumn{3}{c}{CIFAR10} \\
                        & FPR95 & AUROC   & ACC     & FPR95   & AUROC     & ACC      \\ \midrule
100                     & 4.79    & 98.87   & 73.30   & 2.56     & 99.40     & 96.34   \\
200                     & 4.17    & 99.07   & 72.86   & 2.37     & 99.45     & 96.28   \\
300                     & 4.02    & 99.05   & 72.14   & 2.36     & 99.45     & 96.28   \\
400                     & 3.82    & 99.05   & 72.19   & 2.33     & 99.46     & 96.30   \\
500                     & 3.60    & 99.09   & 72.28   & 2.31     & 99.56     & 96.30    \\ \bottomrule
\end{tabular}%
}
\caption{Detailed results of FPR95$\downarrow$ (\%), AUROC$\uparrow$ (\%) and ACC$\uparrow$ (\%) in different training epochs.}
\label{tab:details_epoch}
\end{table}

\section{Detailed Results on Different Ratios $|\mathcal{B}_{\text{in}}^{\text{train}}|/|\mathcal{B}_{\text{wild}}|$}

Table \ref{tab:details_ratios_std} reports the detailed results in different ratios on representative standard benchmark Places and Textures, while Table \ref{tab:details_ratios_hard} reports the detailed results on CIFAR10 and TinyImageNet. As the ratio $|\mathcal{B}_{\text{in}}^{\text{train}}|/|\mathcal{B}_{\text{wild}}|$ increasing, the performance of our method consistently improves.

\section{Detailed Results on the Impact of Epoch in Early-learning Succeeds}
Table \ref{tab:details_epoch} reports the detailed results on the impact of training epochs in early-learning success. The results demonstrate a consistent performance improvement in our LoD model as the number of training epochs increases from 100 to 500.

\end{document}